\documentclass[]{article}

\usepackage[utf8]{inputenc}
\usepackage[T1]{fontenc}
\usepackage{cite}
\usepackage{color}
\usepackage{amsmath}
\usepackage{amsfonts}
\usepackage{amssymb}
\usepackage{graphicx}
\usepackage{amsthm}
\usepackage{geometry}
\usepackage{float}
\usepackage[font=small]{caption}
\usepackage{subcaption}
\usepackage{mathtools}
\usepackage[section]{placeins}
\usepackage{sidecap}  

\graphicspath{{figures/}}

\newtheorem{lemma}{Lemma}
\newtheorem{theorem}{Theorem}

\usepackage[bookmarks=true,hidelinks]{hyperref}

\title{\LARGE \bf  Computing Motion Plans for Assembling Particles with Global Control\footnote{This is the full version of a paper that appears at IROS 2023~\cite{blumenberg2023computing}}}
\author{Patrick Blumenberg\footnote{Department of Computer Science, TU Braunschweig,  Germany. \tt{\{blumenbe, aschmidt\}@ibr.cs.tu-bs.de}}\, , Arne Schmidt$^{\dagger}$, 
Aaron T. Becker$^\dagger$\footnote{Department of Electrical Engineering, University of Houston, USA. \tt{ atbecker@central.uh.edu}}
\thanks{This work was supported by
the Alexander von Humboldt Foundation and the National Science Foundation under \href{https://www.nsf.gov/awardsearch/showAward?AWD_ID=1932572}{CNS 932572} and
\href{https://www.nsf.gov/awardsearch/showAward?AWD_ID=2130793}{IIS 2130793}.}}

\begin{document}

\maketitle
\begin{abstract} 
 We investigate motion planning algorithms for the assembly of shapes in the \emph{tilt model} in which unit-square tiles move in a grid world under the influence of uniform external forces and self-assemble according to certain rules. 
 We provide several heuristics and experimental evaluation of their success rate, solution length,  runtime, and memory consumption.
\end{abstract}

%!TeX root=../fullVersion.tex
\section{Introduction}
In the \emph{tilt model} micro particles move under a global control such as gravity or a magnetic force. 
On actuation, particles move to the designated direction unless blocked by walls or other particles. Compatible particles bond on contact. As an example, see Fig.~\ref{fig:example_instance}.
Previous work \cite{Becker2013, Becker2017, zhang2017rearranging, Schmidt2018, BalanzaMartinez2020, Balanza-Martinez2019} considered the assembly problem that asks whether a given shape can be assembled in the tilt model.
However, these works are mostly theoretical and provide hardness proofs and runtime complexity.
In this paper, we search for efficient motion plans to assemble shapes in a given environment.
In particular, we provide several heuristics and an evaluation based on randomly-generated instances.

\subsection{Related Work}

Motion planning is a fundamental and well-studied research topic in the field of robotics. 
The general problem, known as the \emph{reconfiguration problem}, is to find a sequence of moves of one or multiple robots to transform an initial configuration, i.e., the set of all position of robots and obstacles, to a specific target configuration while avoiding collisions. 
Even when all robots are rectangular, this problem is PSPACE-hard~\cite{Hopcroft1984}.
To solve problems of this type, several heuristics have been developed: 
\emph{search-based} planning algorithms~\cite{Bonet2001} such as A*, which build a graph structure over the configuration space and then use graph-search algorithms to find a path;
\emph{potential field method}, in which an artificial potential field represents attractive and repulsive forces that guide the path of the robots~\cite{Yun1997};
\emph{sampling-based} methods like probabilistic roadmaps (PRM)~\cite{Kavraki1996} and rapidly exploring random trees (RRT)~\cite{Lavalle1998} that can be used to discover a path in high-dimensional configuration spaces (e.g. humanoid robots or multi-robot systems).

In this paper, we consider search-based algorithms as well as RRTs to solve motion planning problems in the tilt model.

\begin{figure}[H]
\centering
\includegraphics[scale=.5,trim=0cm 15.3cm 0cm 0cm, clip]{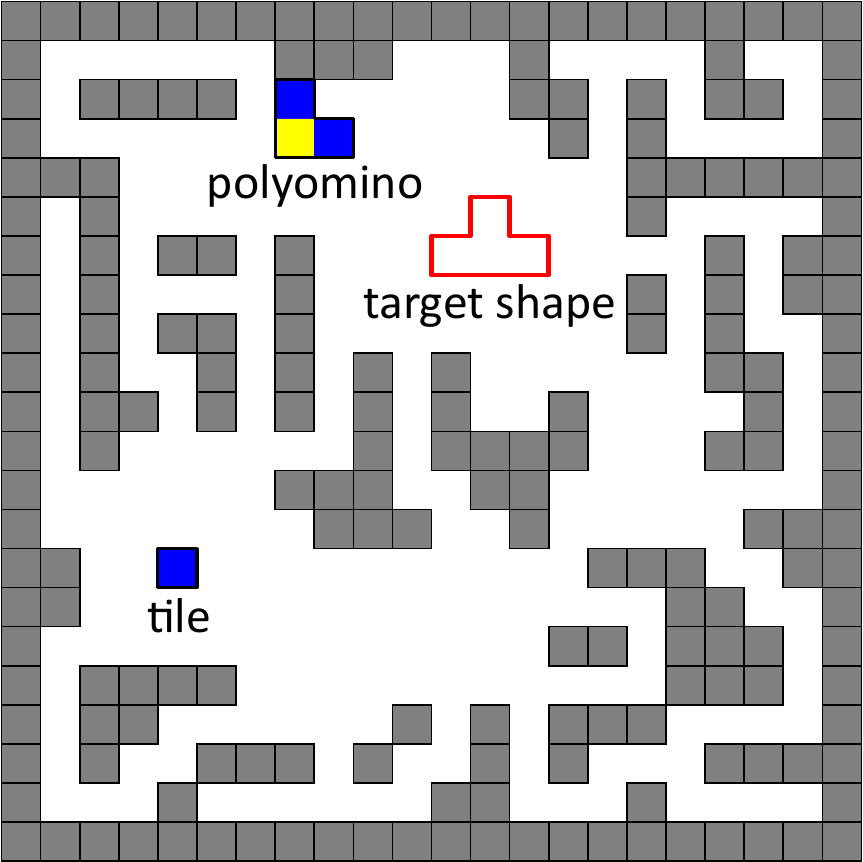}
\includegraphics[scale=.5,trim=0cm 0cm 0cm 15.2cm, clip]{figures/example_boards/fig1_colored.pdf}
\caption[Example instance of the Polyomino Assembly Problem]{Two example instances of the Polyomino Assembly Problem. The blocked positions on the board are shown as grey squares and the tiles as colored squares.
Colored sides correspond to glue types. 
}
\label{fig:example_instance}
\end{figure}

The tilt model of motion planning, introduced in~\cite{Becker2013}, presents a possible solution for scenarios in which a collection of particles or robots that cannot be moved individually needs to be reconfigured. 
When an external force is switched on, all particles move maximally (full tilt motion model as used in~\cite{Becker2017}) or one unit step (single-step tilt model as used in~\cite{Keller2021,Becker2020a}) into the given direction of the force.
In this model, various types of problems have been investigated:
\emph{Assembly problem}, i.e., can a given shape be constructed~\cite{Becker2017, Schmidt2018, Balanza-Martinez2019, BalanzaMartinez2020,Keller2021}; 
\emph{Gathering problem}, i.e., how fast can particles collected in a specific area~\cite{Becker2020a,konitzny2022gathering}; 
\emph{Occupancy Problem}, i.e., can a specific location be reached by at least one particle~\cite{BalanzaMartinez2020, Caballero2020a}; 
\emph{Relocation Problem}, i.e., can a specific particle be moved to its designated position~\cite{Balanza-Martinez2019, Balanza-Martinez2019a, Caballero2020b}; and \emph{Reconfiguration Problem}, i.e., can a configuration be reached where every particle is in its goal position~\cite{Becker2014,Balanza-Martinez2019}. 
Many of the tilt problems are at least NP-hard. In particular, the reconfiguration problem is PSPACE-complete~\cite{Balanza-Martinez2019}.
If the workspace can be designed in advance, it can be shown that rearranging shapes can be performed efficiently~\cite{caballero2021fast,zhang2017rearranging}.

To the best of our knowledge, no other work on the design and experimental evaluation of tilt-related motion planning algorithms in the context of shape assembly has been published. 
Our simulation code and test instance generation are shared at GitHub\footnote{ \href{https://github.com/RoboticSwarmControl/2023TiltMotionPlan/}{https://github.com/RoboticSwarmControl/2023TiltMotionPlan/}}.

\section{Preliminaries}

A \textbf{\emph{board}} $B$ is a rectangular region of $\mathbb{Z}^2$ with each position either marked as \emph{open} or \emph{blocked}. Blocked positions represent obstacles and cannot contain tiles. Furthermore, the induced grid graph $G_{B}$ is defined as the graph in which the open positions of $B$ are nodes and two nodes are connected if and only if their distance is 1. $B$ is called \emph{connected} if and only if $G_{B}$ is connected.

A \textbf{\emph{tile}} $t= (p, g)$ is a unit square centered on open position $p$ with edge \textbf{\emph{glues}} $g := (g_{N}(t),$ $g_{E}(t), g_{S}(t), g_{W}(t))$ of glues. $g_{d}(t)$ denoting the glue on the edge of $t$ facing in the cardinal direction $d \in \{N,E,S,W\}$. Each glue is an element of a finite alphabet $\Sigma$, which also contains a special \texttt{null} glue. A glue function $G \colon \Sigma \to \{0, 1\}$ defines which glues stick to each other. Two glues $g_{1}, g_{2} \in \Sigma$ stick to each other with respect to $G$ if and only if $G(g_{1}, g_{2}) = 1$. Furthermore, $G$ has the properties $G(g_{1}, g_{2}) = G(g_{2}, g_{1})$ and $G({g_{1}, \texttt{null}}) = 0$ for all $g_{1}, g_{2} \in \Sigma$. Tiles \emph{bond} if and only if they are located on adjacent positions and the glues on the shared edge stick together. The bond is considered permanent, and the tiles hereafter move as a unit.

A \textbf{\emph{polyomino}} $P$ is a finite set of tiles that forms a connected component in the graph in which bonded tiles are adjacent.  
The position of $P$ is defined as the position of the top-leftmost tile in $P$.

A \textbf{\emph{configuration}} is a tuple $C = (B, T)$, consisting of a board $B$ and a set of tiles $T$ on open positions. 

A \textbf{\emph{step}} is a transformation between two configurations moving every polyomino (including single tiles) by one unit in one of the cardinal directions unless blocked by an obstacle.
Polyominoes that do not move are called \emph{blocked}.
For simplicity, a step can be denoted by the direction $d \in \{N, E, S, W\}$ in which the polyominoes are moved. 
A \emph{step sequence} is a series of steps.

Given a configuration $C= (B, T)$ and a connected set of open positions $X \subseteq B$, the \textbf{\emph{Polyomino Assembly Problem}} (PAP) asks for either a step sequence $S$ such that consecutively applying the steps of $S$ to $C$, results in a configuration in which a polyomino $P$ exists that satisfies $X = \{p_{1}, p_{2},\ldots \}$, where $p_{1}, p_{2},\ldots$ are the positions of the tiles in $P$, or a proof that the configuration is unreachable.
If there is a fixed tile $t_{\text{fixed}} \in T$ that does not move under step transformations, and tiles can only bond if the resulting polyomino contains $t_{\text{fixed}}$, then we call this problem the \textbf{\emph{Fixed Seed Tile Polyomino Assembly Problem}} (FPAP).

Based on a reduction used in~\cite{Caballero2020a}, we can show that FPAP is PSPACE-complete, even if there are exactly as many tiles on the board as required for the target shape.

\begin{lemma}
	The Fixed Seed Tile Polyomino Assembly Problem is in PSPACE
\end{lemma}

\begin{proof}
	Repeated, non-deterministic selection of control inputs until the target shape is assembled solves the problem and only requires us to store the current configuration. Therefore, the problem is in NPSPACE, which is equal to PSPACE.
\end{proof}

\begin{lemma}
	The Fixed Seed Tile Polyomino Assembly Problem is PSPACE-hard
\end{lemma}

\begin{figure}
	\centering
	\includegraphics[width=\textwidth]{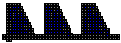}
	\caption[Reduction of the $k$-region relocation problem to Fixed Seed Tile Polyomino Assembly Problem]{Reduction of the $k$-region relocation problem to the Fixed Seed Tile Polyomino Assembly Problem. The red square is the fixed seed tile, the yellow tile $t_{\text{reloc}}$ needs to be relocated $3$ positions to the right in order to assemble the target shape. The blue squares are other tiles. Figure recreated and adapted from \cite{Caballero2020a}.}
	\label{fig:hardnessproof}
\end{figure}

\begin{proof}
	To show PSPACE-hardness we adapt a proof from ~\cite{Caballero2020a} and reduce from the \emph{$k$-region relocation problem} of the full-tilt model~\cite{BalanzaMartinez2020}. In this problem, $k$ disjoint regions each contain a single tile. The goal is to move all $k$ tiles to the $1 \times 3$ region at the bottom of their respective component using tilt transformations. Analogous to \cite{Caballero2020a}, we start the reduction by connecting the $1 \times 3$ regions to a single bottom row and invert tile placement by placing tiles on all open spaces that did not contain a tile in the original instance of the $k$-region relocation problem. Additionally, we place a fixed seed tile $k$ positions to the right and $1$ down from the leftmost tile in the bottom row $t_{\text{reloc}}$, as shown in Figure \ref{fig:hardnessproof}. We assign glues as follows: The fixed seed tile has glue $A$ on all four edges. $t_{\text{reloc}}$ has glue $B$ on all edges. Every other tile has glue $C$ on all edges. Furthermore, we define the glue function $G$, such that $G(A,B) = G(B,C) = G(C,C) = 1$ and $G(X, Y) = 0$ for all other pairs of glues $X, Y$. Finally, we define the target shape as all open spaces inside the connected region, except for the $k$ leftmost positions in the bottom row. Now a solution to the constructed instance of the Fixed Seed Tile Polyomino Assembly Problem corresponds to a solution to the original instance of the $k$-region relocation problem. The key idea is, that tiles can only bond if they are connected to the seed tile and that only $t_{\text{reloc}}$ can directly bond with the seed tile according to the glue function. Therefore, $t_{\text{reloc}}$ must be relocated $k$ positions to the right, which was shown to solve the original $k$-region relocation problem (in \cite{Caballero2020a}). Conversely, once $t_{\text{reloc}}$ has been successfully moved $k$ spaces to the right, the target shape is immediately assembled, as all open spaces, except the $k$ leftmost positions in the bottom row, are filled with tiles that can bond according to the glue function and are connected to the fixed seed tile via $t_{\text{reloc}}$.
\end{proof}

\begin{theorem}
	The Fixed Seed Tile Polyomino Assembly Problem with extra tiles is PSPACE-hard.
\end{theorem}

As a basis for the definition of our heuristic functions, we use the following definitions of distance between positions, tiles and the target shape.
Consider a configuration $C= (B, T)$ and let $p = (p_1, p_2), q = (q_1, q_2)$ be open positions in B.
Then $d_M(p)$ denotes the length of a shortest path from $p$ to any position from $M\subseteq B$ in $G_{B}$,
and
	$d_{1}(p, q) \coloneqq |p_1 - q_1| + |p_2 - q_2|$ is called the taxicab distance.
For the sake of brevity, we define $d(t) \coloneqq d(p_t)$ and analogously $d_1(t,q) := d_1(p_t,q)$ for $t=(p_t, g_t) \in T$.
\section{Algorithmic Approaches}\label{sec:Algorithms}
In this section, we investigate algorithmic approaches for solving the Polyomino Assembly Problem with and without fixed seed tiles. We focus on best-first search algorithms with different heuristics, some of which are admissible, and therefore lead to an A* search that finds solutions of optimal length. 
These heuristics are then further divided into \emph{simultaneous construction}, allowing multiple subassemblies, and \emph{incremental construction}, adding single tiles to the target shape while maintaining separation of the remaining tiles.
We also investigate pruning techniques to avoid unnecessary computation. In addition to the best-first search algorithms, we propose an approach using RRTs.

\subsection {Simultaneous Construction}
Simultaneous construction uses a best-first search with a heuristic based on the distance of the available tiles.

\subsubsection{Greatest Distance heuristic (GD)}

The GD is an admissible heuristic that can be used with an A* search. In this case, the value of a heuristic function $h_{\text{GD}}(C, X)$ is added to the distance of the current configuration from the initial configuration, i.e., the number of moves required to get to the current configuration.
In particular, $h_{\text{GD}}(C, X)$ is defined as follows.
Let $C = (B, T)$ be a configuration and $X$ a target shape with $|X| = n$. 
Then $h_{\text{GD}}(C, X)$ is the distance of the $n$-th nearest tile in $T_{\text{available}}$ to $X$, where $T_{\text{available}}$ is the set of tiles which are part of some polyomino $P$ that can be moved to a position where it is contained in $X$.

In order to decide if a given tile is in $T_{\text{available}}$, our implementation remembers for each polyomino if it is able to reach the target shape and fit into the target shape. These properties are only reevaluated whenever the polyomino changes.
GD is a consistent heuristic because a single step, which has a cost of $1$, applied to a configuration reduces the heuristic value of GD by a maximum of $1$.

$h_{\text{GD}}$ can be used in combination with a greedy best-first search approach by not adding the distance from the initial configuration to the heuristic value. If the heuristic is used this way, we refer to the resulting algorithms as Greedy Greatest Distance (GGD). The idea of this approach is to give up the optimality of the solution in order to potentially improve the execution speed.

\subsubsection{Pruning methods}
There are multiple pruning methods that can be used together with the simultaneous construction approach. Firstly, if $|T_{\text{available}}| < n$, the configuration can never lead to a solution.
Furthermore, when it can be shown that there is no subset of existing polyominoes which can exactly cover the target shape, the branch can be pruned.
In general, however, it may not be computationally feasible to decide whether such a tiling of the target shape exists whenever a polyomino changes.
In case there are exactly as many tiles on the board as required for the target shape, we determine if the $k$ largest polyominoes can be packed into the target shape. If this is not possible then the branch can be pruned. In our implementation, $k$ is set to $3$.

Alternatively, when only one single polyomino exists on the board, we check if this polyomino fits into the target area and if it can be moved to the target area.
Although this may not yield optimal solutions (even if a A* search with a constant heuristic is used), the solution is at most $d_{max}$ steps longer, where $d_{max}$ is the maximum distance between any two open positions.

\subsection{Incremental Construction}

In contrast to simultaneous construction, incremental construction uses multiple consecutive best-first searches to add tiles to a polyomino one after another. For that reason, the heuristic function used in each of the best-first searches depends on the positions of the subassembly and the tile that we are currently trying to add. To make this possible, an incremental construction method needs to determine a suitable set of tiles and a building order in which these tiles can be added to the target polyomino. Each best-first search attempts to keep all tiles that are not involved in the current construction step separated from each other by pruning configurations with undesired subassemblies. Importantly, this approach is not a complete solution, since it is not always possible to avoid the creation of multiple subassemblies. Furthermore, an obtained solution is generally not of optimal length, even if each construction step consists of an A* search with a consistent heuristic.

\subsubsection{Minimum Moves to Polyomino heuristic (MMP)}
Given a configuration $C = (B, T)$, a target shape $X$, and a Polyomino $P$, the MMP heuristic provides a lower bound on the number of moves required to move a selected tile $t= (p_t, g_t)$ to a 
position with coordinates $x$ adjacent to $P$.
Although $x$ may not be an open position, we consider $x$ to be one for the heuristic function, which is defined as follows.
\begin{equation*}
h_{\text{MMP}}(C, X, t, x) \coloneqq \left\lceil \dfrac {d(t, x) - d_{1}(t, x)}{2} \right\rceil+ d_{1} (t, x),
\end{equation*}
i.e., $h_{\text{MMP}}$ is a lower bound on the number of moves required to move $t$ to the target position relative to $P$.

A disadvantage of this heuristic is that it requires knowledge of the pairwise distance between all pairs of the $m$ open positions on the board. These distances can be computed in advance by breadth-first searches starting on each open position in time $O(m^2)$.

\subsubsection{Minimum Moves to Polyomino or Target heuristic (MMPT)}
The MMPT heuristic aims to reduce the need to calculate pairwise distances. To achieve this, MMPT only uses the estimated number of required moves for the heuristic value calculation if the current subassembly and the selected tile are within a specific distance (threshold distance $d_t$) from the target area. Define the target area $A$ as the set of all open positions that are reachable (i.e., can be covered) by the final polyomino. The target area can be computed with a breadth-first search. 
If $t$ or $x$ is not near the target area, the maximum distance of the two components to the target area is used as a basis for the heuristic value instead. To prioritize configurations where both components are near the target area, we apply a weighting factor $w = |A|$ case. 
\begin{equation*}
h_{\text{MMPT}}(C, X, t, x) \coloneqq 
\begin{dcases}
h_{\text{MMP}}(C, X, t, x), \\
    \qquad\text{if } d_{A}(p_t) \leq d_t \land d_{A}(x) \leq d_t\\
w \cdot \max(d_{A}(p_t), d_{A}(x)),\\
    \qquad\text{otherwise}
\end{dcases}
\end{equation*}

This approach has two main advantages over MMP. Firstly, it only requires the computation of pairwise distances within the proximity of the target area. Secondly, it encourages the formation of subassemblies within the reachable area of the target shape. \par

\subsubsection{Distance to Fixed Position heuristic (DFP)}
In the case of a fixed seed tile, a much simpler heuristic can be used.
Because the target position of the tile is fixed, each move can reduce the length of a shortest path from the tile to the target position by at most $1$. 
Therefore:
\begin{equation*}
h_{\text{DFP}}(C, X, t, x) \coloneqq d(t, x)
\end{equation*}
Additionally, during the computation of the shortest paths to the target position the positions of tiles contained in the fixed polyomino, as well as neighboring positions that are impassable for the selected tile because of glues on the edges of a fixed tile, are marked as blocked.

\subsubsection{Pruning methods}
All of the incremental construction approaches use the same two pruning methods. Firstly, a branch is pruned if its configuration contains more than one subassembly. Secondly, a branch is pruned if the existing subassembly cannot be moved to a position where it is contained in the target shape.

\subsubsection{Computation of the building order}
A brute-force algorithm is used to determine a tiling of the target shape such that all tiles bond. A recursive backtracking algorithm is then utilized to remove tiles from the polyomino one at a time, in order to determine a potential building order. 
In each iteration, the algorithm checks if the remaining tiles are still connected by glues and if a path exists along which the selected tile can be moved outside of the enclosing rectangle without being blocked by other tiles or glues.
If all tiles can be removed from the polyomino in this way the reversed deconstruction order is a \emph{potential} building order.
This approach is directly inspired by research on the constructibility of polyominoes under tilt transformations \cite{Becker2017}.
If no building order is found, a different tiling of the polyomino is attempted until a potential building order is found or all possible tilings were considered.
A building order found in this way is not guaranteed to be realizable on the board of the problem instance. If any of the best-first searches during the motion planning process terminate without finding a solution we proceed with the next candidate building order.

\subsection{Rapidly-expanding Random Tree (RRT)}
The basic idea of an RRT is to select a random configuration from the configuration space and expand the configuration in the current tree that is closest according to some \emph{cost-to-go} function towards the selected configuration.
When designing an RRT-based algorithm for the Polyomino Assembly Problem, the main challenge is to find a cost-to-go function that is fast to compute and provides a reasonable estimate of the distance between two configurations.

Let $C_1 = (B, T_1)$ and $C_2 = (B, T_2)$ be two configurations on the same board. 
For two tiles $t_1 = (p_1, g_1) \in T_1$ and $t_2 = (p_2, g_2) \in T_2$ we define
\begin{equation*}
d_H(t_1, t_2) \coloneqq
\begin{cases}
d(p_1, p_2) , & \text{if } g_1 \equiv g_2 \\
\infty , & \text{otherwise.}
\end{cases}
\end{equation*}
Then we define the distance between a tile $t$ and a set of tiles $S$ as $D_H(t, S) \coloneqq \min_{s \in S} d_H(s, t)$.
Note that $D_H(t_1, T_2) < \infty$ for any $t_1 \in T_1$, and vice versa, because both sets of tiles have the same sets of glues. 

Now the cost-to-go-function between $C_1 = (B, T_1)$ and $C_2 = (B, T_2)$ is defined \mbox{equivalently} to the Hausdorff distance \cite{hausdorff} between the two sets of tiles.
\begin{equation*}
D_H(C_1, C_2) \coloneqq \max \left\{	\max_{t_{1} \in T_{1}} d(t_1, T_2) , \max_{t_{2} \in T_{2}} d(t_2, T_1) \right\}
\end{equation*}

Additionally, if a polyomino exists in $C_1$ that does not fit into any polyomino of $C_2$ we set the cost-to-go to $\infty$. 

Since the number of times the cost-to-go function is evaluated in each expansion step grows with the number of nodes in the RRT, we try to further reduce the computational cost of repeatedly evaluating the cost-to-go function by keeping a sparser tree. This is achieved through a more expensive expansion step, which expands the closest node by multiple steps in the direction of the randomly selected configuration. For that purpose, a greedy best-first search with a limited number of iterations is used to minimize the cost-to-go. 

A downside of this cost-to-go function is that it requires knowledge about the pairwise distances between positions of the board. A potential solution is to use the taxicab distance instead of the length of a shortest path as the underlying distance metric. However, this would further decrease the accuracy of the estimated distance. 

To estimate the distance to the goal, i.e., any configuration containing the target shape, a different metric must be used because the order of tiles within the target shape and the position of possible left-over tiles are not known. Therefore we choose the greatest distance among the $n$ nearest available tiles, as defined by the GD approach, as an estimate for the distance to the goal. We bias the RRT search to expand the closest viable node directly towards the goal in $5\%$ of the expansion steps. For these expansion steps, the GGD approach with a limited number of iterations is used. If the distance from the selected node to the goal cannot be reduced in this way, the node is marked and the next expansion towards the goal will use the next closest node. 
\section{Simulation Setups}\label{sec:SimulationSetup}
The algorithmic approaches are evaluated on sets of procedurally generated instances with different features, including instances with and without a fixed seed tile. 
We investigate the impact of  the number of tiles, size of the board, obstacle placement, etc., on the problem difficulty.
The analysis focuses on the success rate, runtime, and length of the solutions. 

\subsection{Method for Instance Creation}
 We implemented a Python function that procedurally generates instances based on six input parameters.
 
 \begin{description}
 	\item[Board type] Two board types are used:  \emph{Maze} starts on a board filled with blocked positions. A randomized recursive backtracking algorithm creates a tree of open spaces. Next, a number of rectangular open regions proportional to the size of the board are added uniformly at random. 
 	\emph{Cave} boards are created by cellular automata. In the next generation a dead cell becomes alive if it has five or more living neighbors;  a living cell becomes dead if it has less than four living neighbors. Each cell starts alive with a probability of $0.45$ and two generations are applied.  
 	Living cells are blocked positions. 
 	To ensure connectedness, the largest connected component of open spaces is determined and all other open spaces are filled.
 	\item[Target shape size]
 	The number $n$ of tiles required to assemble the target shape.
 	\item[Board size] The edge length of the board.
 	Only square-shaped boards were used. 
 	\item[Number of extra tiles]
 	Tiles not needed to create the target shape. The total number of tiles is the sum of the target shape size and the number of extra tiles.
	\item[Number of different glues]
 	The glues on the edges of each tile are selected uniformly at random from the set of available glues. To make it likely for instances to have a solution, a subset of the tiles is arranged into the target shape and a set of rules is added such that they bond. Finally, additional rules are selected at random, until at least half of all possible combinations of glues stick together.
	\item[Problem type] Whether a fixed seed tile exists or not. If a fixed seed tile exists, tiles can only bond if the resulting polyomino contains the fixed tile.
 \end{description}

On all generated boards, the open positions form one connected region.
To determine initial tile positions, tiles are successively placed on legal positions uniformly at random. A legal position is open, does not contain a tile, and is not adjacent to a position containing a tile. 
If a fixed seed tile is required, it is placed first on a suitable position within the target shape.  Although some measures were taken to increase the probability that the created instances are solvable, instances are not guaranteed to have a solution.

\subsection{Simulation Environment}
All motion planning algorithms were implemented in a self-developed simulator based on TumbleTiles\footnote{TumbleTiles: \url{https://github.com/asarg/TumbleTiles}}. The original TumbleTiles software was upgraded from Python 2 to Python 3 and heavily modified. 
It can simulate step transformations on configurations with and without fixed seed tile.

Experiments were conducted on multiple computers, each with the same specifications (\textbf{Intel\textsuperscript{\textregistered} Core\textsuperscript{TM} i7-6700K CPU @ 4x4.00GHz, 64GB RAM}) running Ubuntu 20.04 LTS. 

\subsection{Evaluated Approaches}
\begin{enumerate}
\setlength\itemsep{0.2em}
\item Breadth-first-search (BFS)
\item Greatest Distance heuristic (GD)
\item Greedy Greatest Distance heuristic (GGD)
\item Minimum Moves to Polyomino heuristic (MMP)
\item Minimum Moves to Polyomino or Target heuristic (MMPT): $d_t = 4$.
\item Distance to Fixed Position heuristic (DFP)
\item Rapidly-exploring Random Tree (RRT):  Each expansion step consists of $40$ iterations of a greedy best-first search.
The bias towards the goal is $5$\%. If within 7 distance to the goal, a best-first search limited to $500$ iterations attempts to find a solution.
\end{enumerate}

For each solver, all applicable pruning methods discussed in Sec.~\ref{sec:Algorithms} were used.
In combination with the simultaneous construction heuristic approaches GD and GGD, the alternative stop condition is used if the configuration contained no extra tiles.

\subsection{Experimental Procedure}
Two instance sets were created in advance and contain the same instances for every experiment. Both contain five procedurally generated instances for each  combination of parameters ($2160$ instances each) from the following table:
\begin{center}
\begin{tabular}{ |l|c|c| }
\hline
\texttt{InstanceSet} & \texttt{1} & \texttt{2}\\
\hline
Board & \multicolumn{2}{c|}{cave, maze} \\
%Board type & cave, maze  & cave, maze \\
\hline
Target shape size & $3 , 4, 5 , 6$ & $5, 10, 13, 15$ \\
\hline
Board size & $20 , 30, 40$ & $40, 80, 120$ \\
\hline
Extra tiles & $0, 1, 3$ & $0, 3, 5$ \\
\hline
Number of glues & $1 , 2, 3$ & $1, 3, 5$ \\
\hline
Problem type & \multicolumn{2}{c|}{$\text{seed tile}, \text{no seed tile}$}  \\
\hline
\end{tabular}
\end{center}

\begin{figure}[t]
	\centering
	\includegraphics[angle=-90, width=.9\columnwidth]{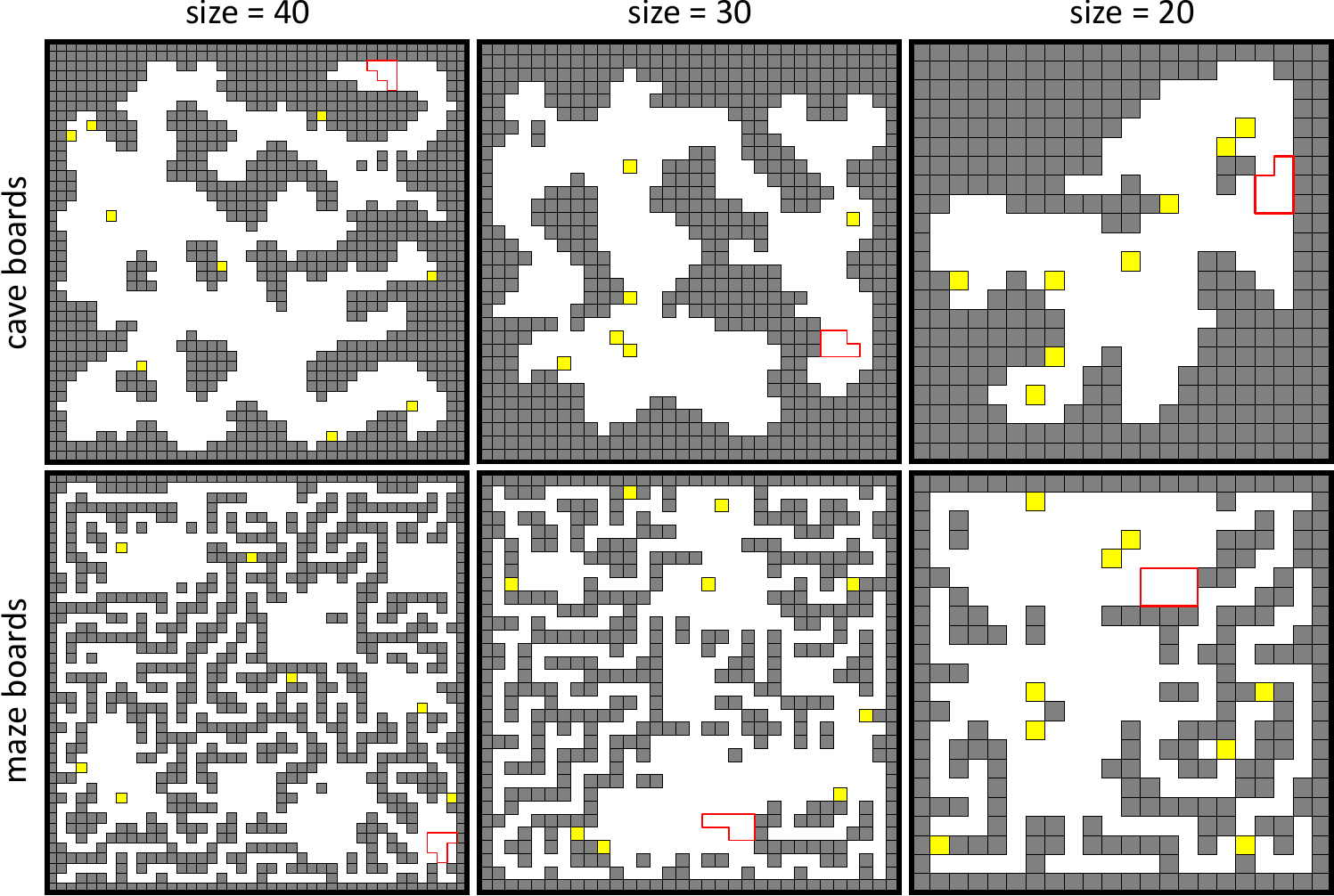}
	\caption{Six instances from \texttt{InstanceSet1}.
		Yellow squares are tiles and the target shapes are outlined in red. Glues are not shown.}
	\label{fig:example_boards}
\end{figure}
\texttt{InstanceSet1} consists of relatively easy to solve instances, that allow a comparison of all solvers, including the breadth-first search solver and other (near)optimal solvers.  Six instances are shown in Fig.~\ref{fig:example_boards}. 
\texttt{InstanceSet2} consists of instances with a wider range of difficulties. See Fig.\ref{fig:example_instance} for an example.
 Only GGD, MMPT, and DFP were used for \texttt{InstanceSet2}. DFP was only used on instances with a fixed seed tile. \par
All experiments were conducted with a timeout of 600 seconds per instance.
The problem instance, solution sequence (if solved), runtime,  peak memory usage, and number of nodes in the search tree (sum of nodes in all search trees for incremental construction) were recorded.

\section{Results of Simulation}\label{sec:Results}

Our algorithmic approaches are evaluated using the experiment results from both sets of instances.
\subsection{Evaluation of \texttt{InstanceSet1}}
The experiment results from \texttt{InstanceSet1} are used to measure the impact of various parameters on the performance of different motion planning algorithms. 
The relatively small size of these instances allows every solver to solve a fraction of the instances within the timeout of 10 minutes.
This data is used to evaluate the impact of configuration parameters on the efficiency of the different motion planning algorithms and compare memory requirements.
Breadth-first search is used as the baseline for performance comparison.

\subsubsection{Comparison of simultaneous construction algorithms}
The performance of the best-first search motion planning algorithms GD and GGD is compared with breadth-first search. 
As expected, all heuristic searches show a better performance in terms of runtime and success rate than BFS (see Fig.~\ref{fig:hs_i1_performance2}). 
The A* search with the consistent heuristic GD takes longer to find a solution and solves fewer instances within the time limit than the greedy best-first search approach GGD. 
All solvers show a decrease in successfully solved instances with an increasing number of tiles.
Compared to other parameters such as board size, the number of tiles has the greatest performance impact. 

Conversely, we observe that the board size has a greater impact on the length of the found solution than the number of tiles.
The near-optimal motion planning algorithm GD produces solutions of approximately the same length as BFS (see Fig.~\ref{fig:hs_i1_performance3}) whereas the solution length of the greedy algorithm was greater by a factor of $3.76$ on average when compared only on instances for which both GD and GGD found a solution.
However, this factor decreases with an increase in board size as Table I indicates.

\begin{center}
\begin{SCtable}
\begin{tabular}{ |l|r| }
\hline
Board size & $\text{GGD} / \text{GD}$ length\\
\hline
$20 \times 20$ & $3.92 \pm 3.97$ \\
\hline
$30 \times 30$ & $3.67 \pm 3.57$ \\
\hline
$40 \times 40$ & $3.46 \pm 2.82$ \\
\hline
\end{tabular}
\caption*{Table I: Mean ratio of lengths of solutions found by GGD and GD to the same instance $\pm \text{ SD}$ for different board sizes.}
\vspace{-1em}
\end{SCtable}
\end{center}

\begin{figure}[htb]
\centering
\includegraphics[width=\columnwidth]{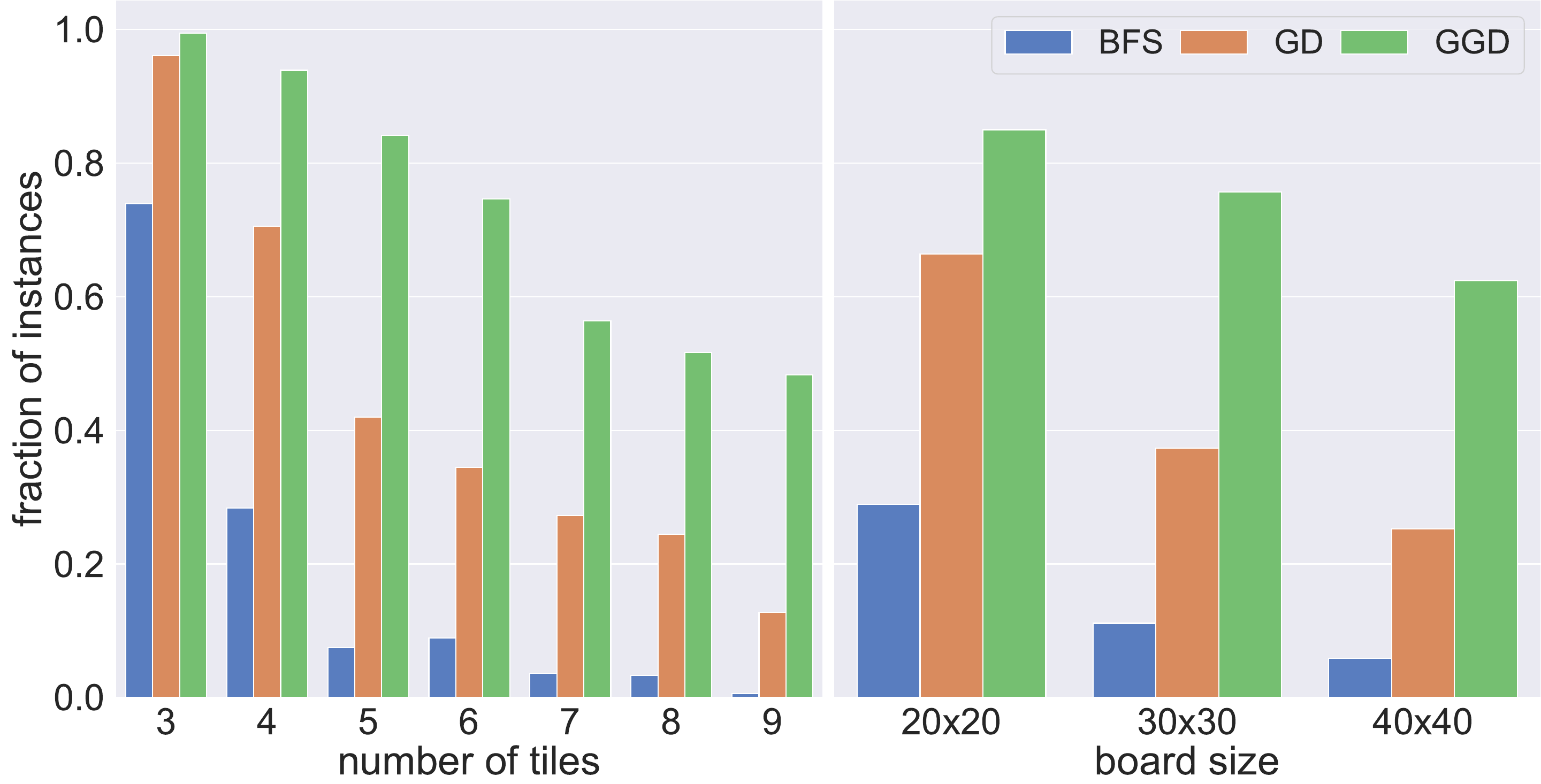}
\caption [Fraction of \texttt{InstanceSet1} solved by the search-based planners] {Fraction of instances from \texttt{InstanceSet1}  solved  by simultaneous construction algorithms.}
\label{fig:hs_i1_performance2}
\end{figure}

\begin{figure}[htb]
\centering
\includegraphics[width=\columnwidth]{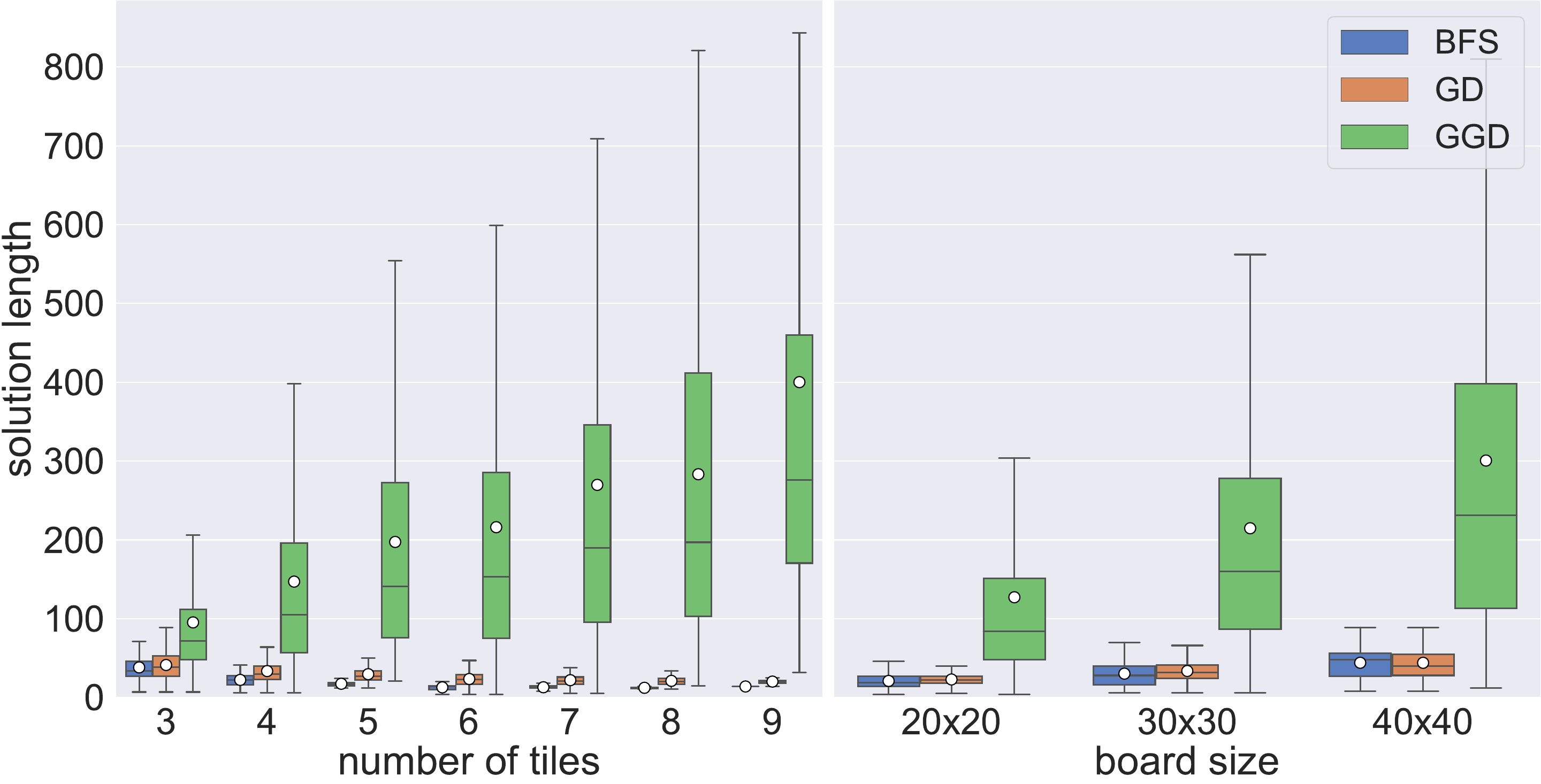}
\caption[Solution lengths for search-based heuristic planners on \texttt{InstanceSet1}]{Solution length of simultaneous construction algorithms on \texttt{InstanceSet1}. Only instances for which a solution was found are taken into account. The boxes show upper and lower quartile, the line the median, the whiskers extend to the most extreme value that is
within a proportion of 1.5 of the interquartile range. The white dot indicates the arithmetic mean. Outliers are omitted for better readability.}
\label{fig:hs_i1_performance3}
\end{figure}

\subsubsection{Comparison with other motion planning algorithms}
This section compares the incremental construction motion planning algorithms (MMP, MMPT, DFP), and RRT. 
For comparison, GGD is included in the plots.
As shown in Figs.~\ref{fig:i1_performance1} and \ref{fig:i1_performance2}, all incremental construction algorithms solve larger fractions of instances successfully than GGD on instances with a number of tiles greater than 5. The runtime on instances with many tiles is clearly better too. 
However, the runtime of MMP is particularly sensitive to an increase in board size.
MMPT does not have this downside and shows superior performance to MMP.
In general, the performance of the incremental approach degrades slower than the performance of other approaches when the number of tiles is increased.

The DFP motion planner can only be evaluated on instances of the Fixed Seed Tile Polyomino Assembly Problem. It has a high success rate of around $90\%$ on these instances regardless of the number of tiles and the size of the board. Furthermore, a large majority of instances that are solved by DFP are solved in less than 1 second. 
RRT displays a similar behavior as GGD but the runtime is more sensitive to increased board size. 

\begin{figure}[htb]
\centering
\includegraphics[width=\columnwidth]{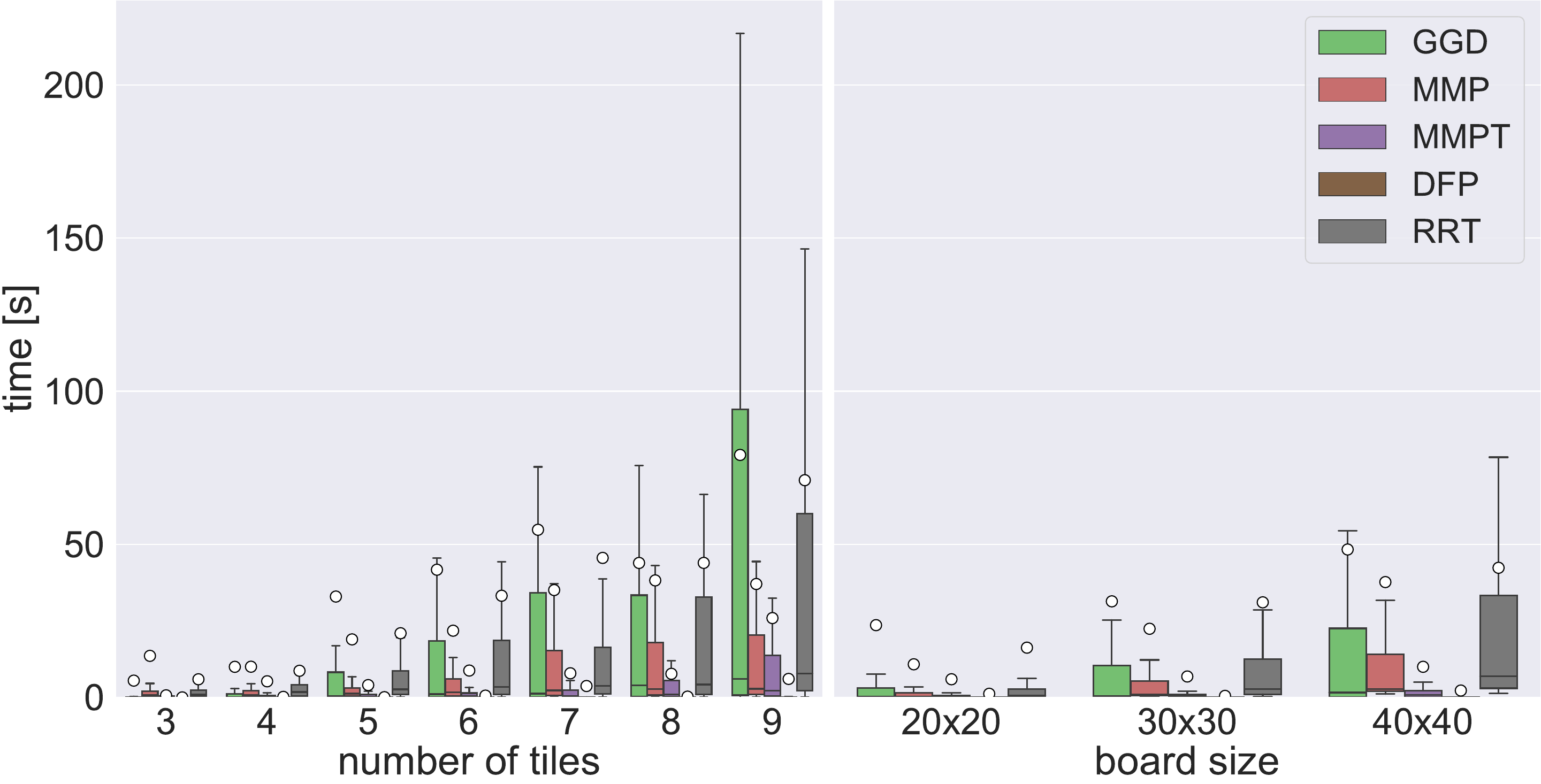}

\caption[Runtime of several planners on \texttt{InstanceSet1}]{Runtime  on \texttt{InstanceSet1} of GGD, incremental construction algorithms, and RRT. Only successful solutions are shown. Outliers are omitted.}
\label{fig:i1_performance1}
\end{figure}

\begin{figure}[htb]
\centering
\includegraphics[width=\columnwidth]{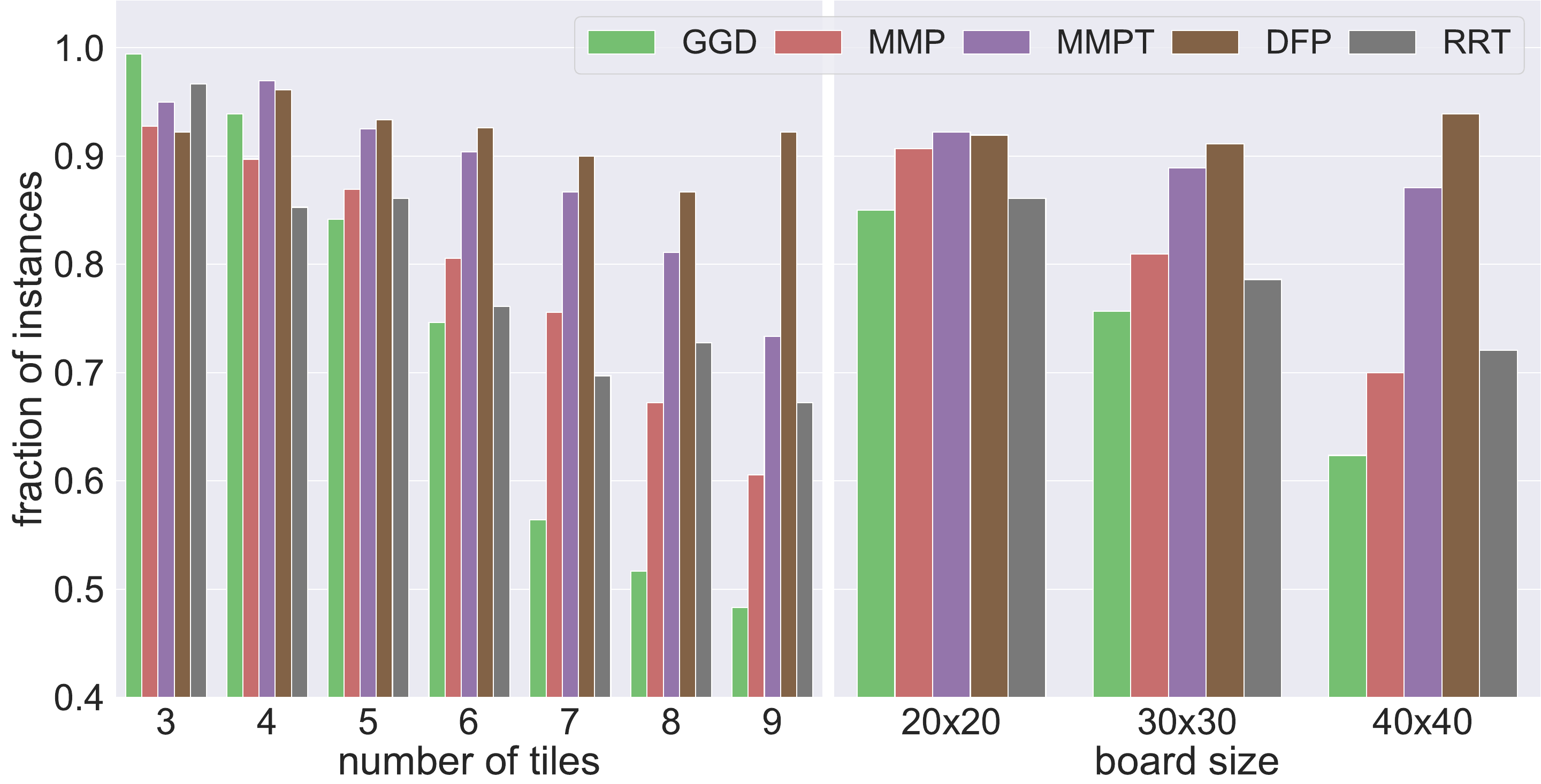}
\caption [Fraction of \texttt{InstanceSet1} solved by several planners] {Fraction of instances from \texttt{InstanceSet1} solved by GGD, incremental construction algorithms, and RRT. For DFP the fraction
of suitable instances with a fixed tile that were solved is shown.}
\label{fig:i1_performance2}
\end{figure}

\subsubsection{Comparison of the difficulty of the two problem types}
\texttt{InstanceSet1} contains instances of both the Polyomino Assembly Problem and the Fixed Seed Tile Polyomino Assembly Problem.
We compared the performance of three fundamentally different motion planning approaches (GGD, MMPT, RRT) for both problems in order to evaluate which problem is harder to solve in practice.
Figure \ref{fig:mmpt_fixed} shows the fraction solved by GGD. These results are similar to RRT. 
All three solvers demonstrate decisively faster runtime for the problem with a fixed seed tile (see Figure \ref{fig:mmpt_fixed_runtime}).
Furthermore, the specialized motion planning algorithm DFP for instances with a fixed seed tile solved a large fraction of the suitable instances in a short time, as Figures \ref{fig:i1_performance1} and \ref{fig:i1_performance2} indicate.

\begin{figure}[htb]
\centering
\includegraphics[width=\columnwidth]{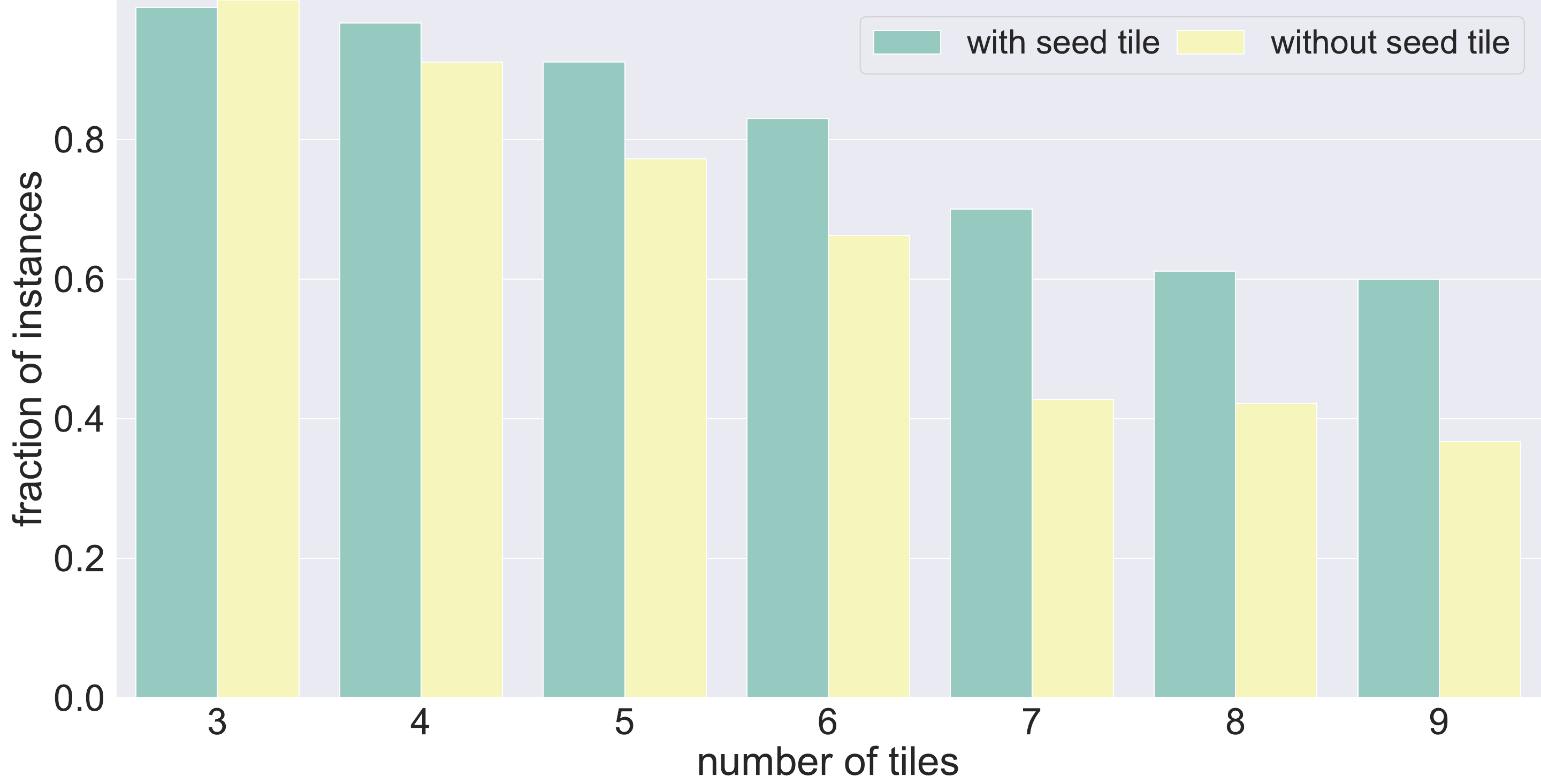}
\caption[Runtime comparison by problem type for GGD]{Comparison of the fraction of instances with and without fixed seed tile solved by the GGD solver.}
\label{fig:mmpt_fixed}
\end{figure}

\begin{figure}[htb]
\centering
\includegraphics[width=\columnwidth]{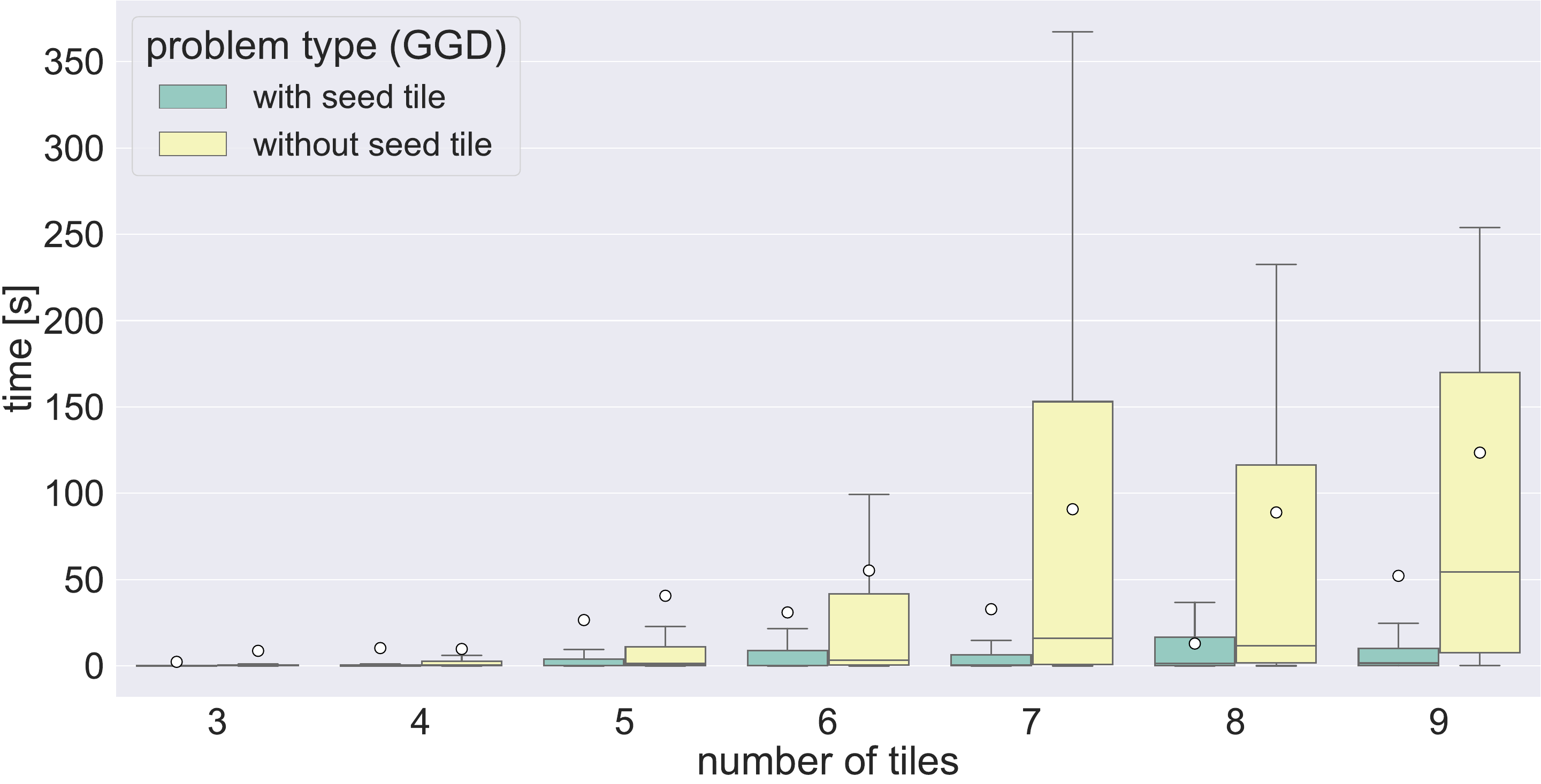}
\caption[Runtime comparison by problem type for GGD]{Comparison of the runtime of instances with and without fixed seed tile solved by the GGD solver.}
\label{fig:mmpt_fixed_runtime}
\end{figure}

\subsubsection{Memory usage}
The peak memory requirement of three motion planning algorithms depending on the time needed to solve an instance are shown in Fig.~\ref{fig:gad_i1_peek_memory_usage_over_time}.
For GGD and MMPT, the maximum of the peak memory usage grows linearly over time, because the speed at which nodes are expanded remains constant over time and the memory requirements per node remain constant. The memory usages of both algorithms cover a similar range with a maximum of around $8$\,GB and $9$\,GB respectively.
In contrast, the peak memory usage of RRT is two orders of magnitude smaller, because only a sparse tree of configurations is stored. Furthermore, it does not increase linearly over time. Instead, it increases sharply in the first few seconds and then slows down significantly. An explanation for this behavior is that every iteration requires the computation of the Hausdorff distance from a random configuration to all configurations in the RRT. As the number of nodes grows over time, each iteration takes more time than the previous one. Furthermore, up to a certain number of iterations the memory requirements for a single expansion step, which includes a heuristic search, outweigh the memory requirements of the RRT.

\begin{figure}[htb]
\centering
\includegraphics[width=\columnwidth]{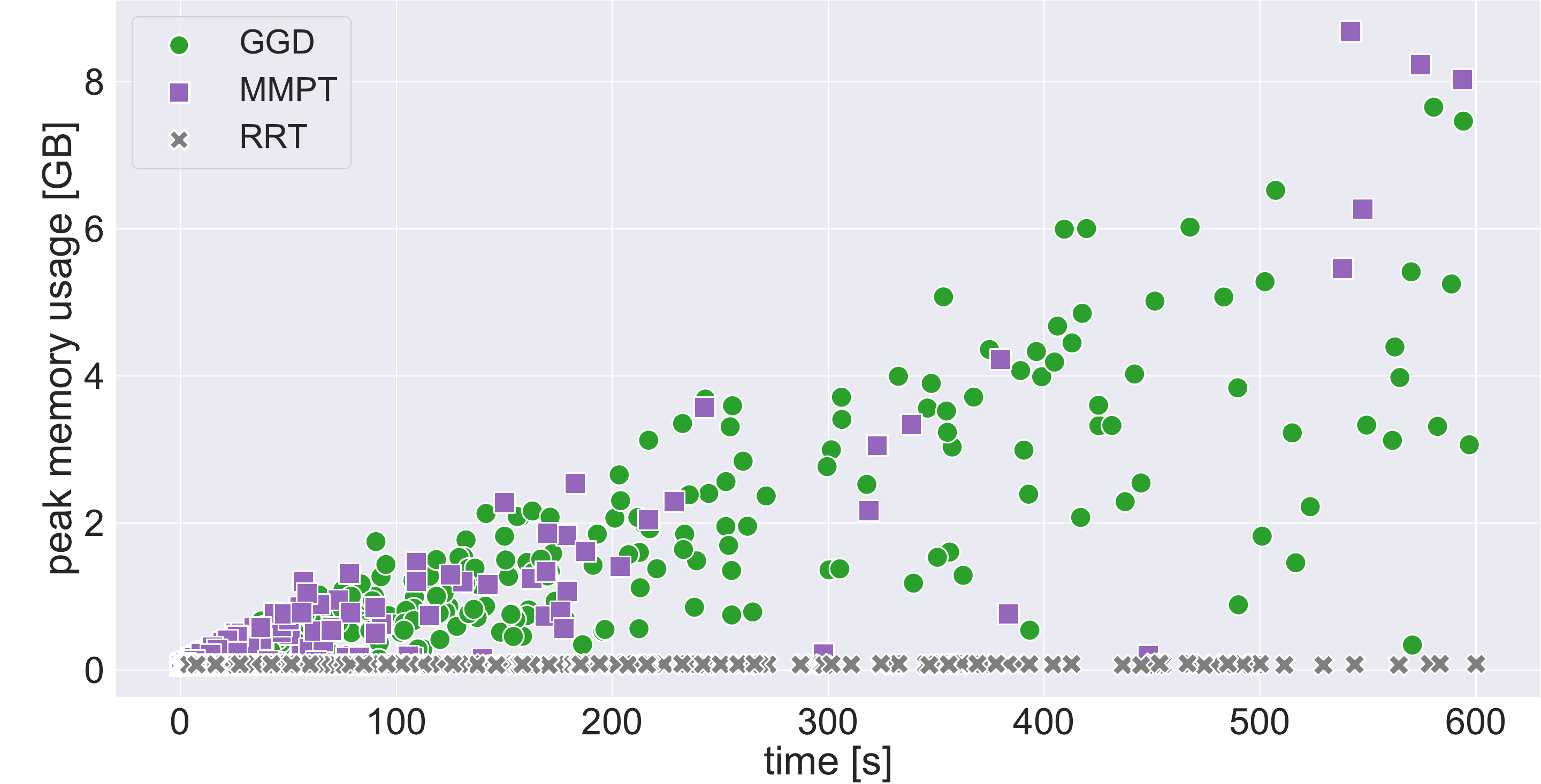}
\caption[Memory usage]{Peak memory usage of GGD, MMPT and RRT in GB depending on the runtime of the solver as scatter plot.}
\label{fig:gad_i1_peek_memory_usage_over_time}
\end{figure}

\subsection{Evaluation of \texttt{InstanceSet2}}
\texttt{InstanceSet1} results show that the near-optimal motion planning algorithms are usually unable to solve larger instances. Regarding the RRT solver, the calculation of the cost-to-go function based on the length of a shortest path is computationally infeasible for the larger boards in \texttt{InstanceSet2}. Therefore, we choose to compare GGD, MMPT, and DFP.

\subsubsection{Comparison of the motion planning algorithms}

The time needed by the motion planning algorithms and their success rate can be seen in Figs.~\ref{fig:i2_performance1} and \ref{fig:i2_performance2}. These plots are grouped by the target shape size, instead of the total number of tiles on the board. Instances can have up to $5$ additional tiles that are not needed to build the target shape. Generally speaking, the runtime of all algorithms increases with an increased size of the target shape. Again, all algorithms finish a majority of the solved instances long before the timeout, which indicates that there are certain instances that the solvers struggle with, whereas other instances can be solved easily.
The performance of GGD in terms of the fraction of solved instances continues to decrease sharply with an increase in the size of the target shape. MMPT shows an overall higher success rate but the performance also falls off when the target shape gets bigger. DFP performs best and shows a high success rate that only declines slowly with an increase in the target shape size and does not fall under $60\%$.  
Furthermore, DFP finds a solution much faster than all other algorithms. Since instances are not guaranteed to have a solution, it is possible that randomly generated instances with a larger target shape are less likely to be solvable, which could also be a factor for the fraction of solved instances.
The greater range of board sizes in \texttt{InstanceSet2} confirms that the runtime increases much faster with an increased number of tiles, rather than increased board size. In particular, the success rate of DFP is not decreased when the board size is increased from $40 \times 40$ to $120 \times 120$.

\begin{figure}[htb]
\centering
\includegraphics[width=\columnwidth]{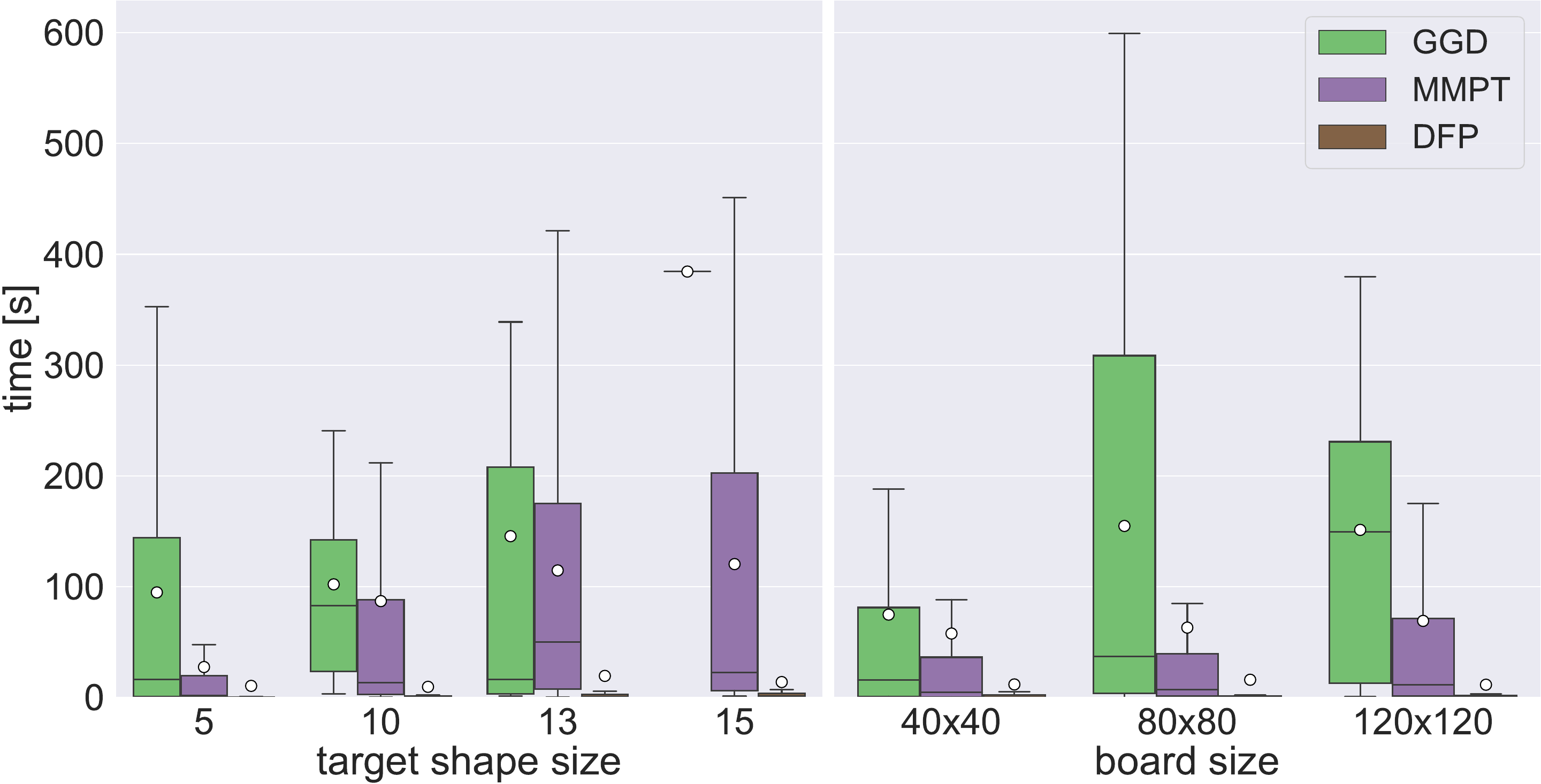}
\caption[Runtime of several motion planners on \texttt{InstanceSet2}]{Comparison of the runtime of three motion planning algorithms on \texttt{InstanceSet2}.  Only successful solutions are shown. GGD only solved 1 instance for target size 15.}
\label{fig:i2_performance1}
\end{figure}

\begin{figure}[htb]
\centering
\includegraphics[width=\columnwidth]{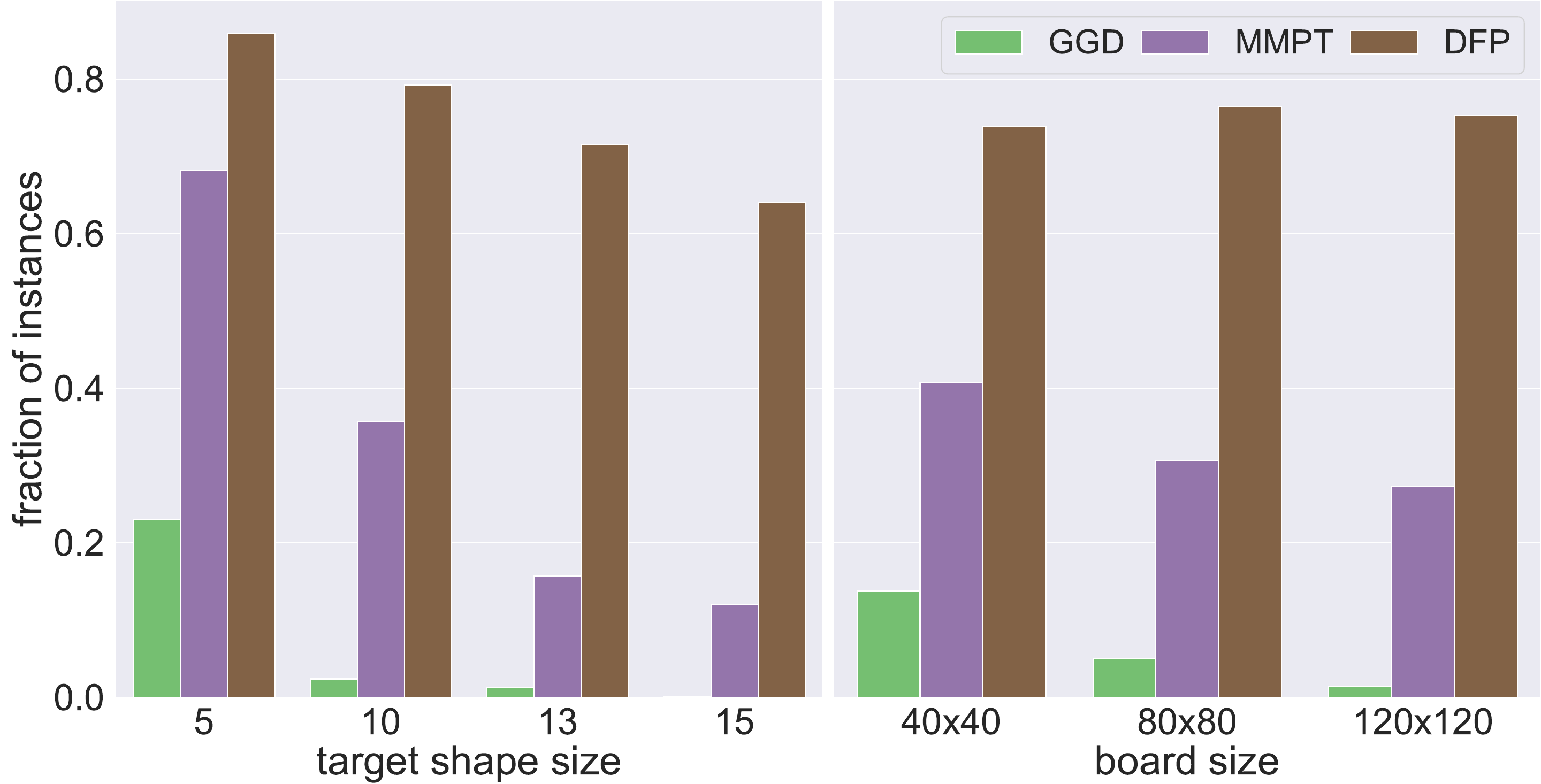}
\caption [Fraction of \texttt{InstanceSet2} solved by several planners] {Fraction of instances from \texttt{InstanceSet2} that were solved by three motion planning algorithms. For DFP the fraction of suitable instances with a fixed tile that were solved is shown.}
\label{fig:i2_performance2}
\end{figure}

\section{Conclusion and Future Work}
We designed motion planning algorithms for the assembly of shapes in the tilt model based on multiple different approaches, including search-based and sampling-based motion planning algorithms. 
We evaluated the effectiveness of these approaches experimentally on procedurally created boards and investigated the parameters that are most significant for the performance of motion planning algorithms in general, as well as the specific strength and weaknesses of the different approaches.
 The best complete algorithm analyzed was GGD. 
 This was outperformed by incomplete algorithms, the best of which was MMPT. All of these were outperformed by DFP, an incomplete motion planner that uses a fixed seed tile.
The evaluation of computational complexity for the Polyomino Assembly Problem without a fixed seed tile and without extra tiles is left to future research. 

RRT is a promising method to solve tilt motion planning problems. Combined with a computationally more expensive expansion step, it has the added benefit of requiring less memory than other approaches. A major challenge in this context is to find a cost-to-go function that is fast to compute and gives a good approximation of the actual distance between two configurations. 
Better best-first search algorithms could potentially be achieved with heuristics that not only depend on the distance of tiles to the target shape but instead consider, for example, the involved glue types and possible positions of tiles within the target polyomino. 

 On the instance side, it seems that the number of moving tiles drastically increases the time to solve. 
This result seems analogous to Schwartz and Sharir's result on moving disks through a polygon~\cite{schwartz1983piano}.
Thus, we expect our problem without a fixed seed tile to be PSPACE-complete. This, however, is left for future work.

\bibliography{biblio}
\bibliographystyle{abbrv}
\end{document}